\documentclass{sig-alternate}

\usepackage{amssymb}
\usepackage{graphicx}
\usepackage{algorithm}
\usepackage{algorithmic}
\usepackage{amsmath}
\usepackage{booktabs}
\usepackage{times}
\usepackage{helvet}
\usepackage{courier}
\usepackage{threeparttable}
\usepackage{array}
\usepackage{url}

\begin{document}
%
% --- Author Metadata here ---
%\conferenceinfo{WOODSTOCK}{'97 El Paso, Texas USA}
%\CopyrightYear{2007} % Allows default copyright year (20XX) to be over-ridden - IF NEED BE.
%\crdata{0-12345-67-8/90/01}  % Allows default copyright data (0-89791-88-6/97/05) to be over-ridden - IF NEED BE.
% --- End of Author Metadata ---

\title{A Convergence Theorem for the Graph Shift-type Algorithms}
%\subtitle{[Extended Abstract]}
%\titlenote{A full version of this paper is available as
%\textit{Author's Guide to Preparing ACM SIG Proceedings Using
%\LaTeX$2_\epsilon$\ and BibTeX} at

\numberofauthors{1} %  in this sample file, there are a *total*
\author{
}
% There's nothing stopping you putting the seventh, eighth, etc.
% author on the opening page (as the 'third row') but we ask,
% for aesthetic reasons that you place these 'additional authors'
% in the \additional authors block, viz.
%\additionalauthors{Additional authors: John Smith (The Th{\o}rv{\"a}ld Group,
%email: {\texttt{jsmith@affiliation.org}}) and Julius P.~Kumquat
%(The Kumquat Consortium, email: {\texttt{jpkumquat@consortium.net}}).}
% Just remember to make sure that the TOTAL number of authors
% is the number that will appear on the first page PLUS the
% number that will appear in the \additionalauthors section.
\newtheorem{definition}{Definition}
\newtheorem{lemma}{Lemma}
\newtheorem{theorem}{Theorem}
\newtheorem{proposition}{Proposition}
\maketitle
\begin{abstract}
Graph Shift (GS) algorithms are recently focused as a promising approach for discovering dense subgraphs in noisy data. However, there are no theoretical foundations for proving the convergence of the GS Algorithm. In this paper, we propose a generic theoretical framework consisting of three key GS components: simplex of generated sequence set, monotonic and continuous objective function and closed mapping. We prove that GS algorithms with such components can be transformed to fit the Zangwill's convergence theorem, and the sequence set generated by the GS procedures always terminates at a local maximum, or at worst, contains a subsequence which converges to a local maximum of the similarity measure function. The framework is verified by expanding it to other GS-type algorithms and experimental results.
\end{abstract}

\category{H.4}{Information Systems Applications}{Miscellaneous}
%A category including the fourth, optional field follows...
\category{D.2.8}{Software Engineering}{Metrics}[complexity measures, performance measures]

\terms{Theory}

\keywords{Convergence Proof; the Graph Shift Algorithm; the Dominant Sets and Pairwise Clustering; the Zangwill's Theorem.} % NOT required for Proceedings

\section{Introduction}

%Graph mining has become a popular research topic in machine learning and data mining. It has also been applied to many important fields, including social network analysis \cite{Chakrabrarti2006Graph} and web mining \cite{kosala2000web}. Among the multiple tasks in graph mining, an important problem is the subgraph mining, in particular the dense subgraph discovery, i.e., graph mode seeking.
%Dense subgraph discovery is an NP-complete problem. This leads to numerous heuristic methods proposed, including the work of Gibson et al \cite{gibson2005discovering} in utilizing shingling technique, GRASP Algorithm \cite{abello2002massive}, Cohesive Subgraph Visualization approach \cite{wang2008csv}, etc.
% , including XXXXXXXXX.
%including one framework called graph shift liked Algorithm, such as pairwise clustering\cite{pavan2007dominant}, and graph shift algorithm\cite{liu2010robust}.
%Among these various methods, Pavan et al \cite{pavan2007dominant} put forward a framework called Dominant Sets and Pairwise Clustering (DSPC) Algorithm, treating the dense subgraph discovery problem as a constrained optimization problem. Liu et al \cite{liu2010robust} further modified it, and proposed a Graph Shift (GS) Algorithm for the graph mode seeking.
%GS algorithms seldom consider the so-called outside data points, including the outliers. It outperforms other graph clustering methods in terms of dealing with noisy data sets. Also, since  GS algorithms are always operated on partial graphs, they save most of computational load and are very efficient.
The Graph Shift (GS) Algorithm\cite{liu2010robust} is a newly proposed algorithm in seeking the dense subgraph (also known as graph mode) and has received many attentions in machine learning and data mining area. As its tremendous advantages in removing the noise points in learning the dense subgraph, it is popularly used in image processing areas such as common pattern matching \cite{liu2010common}\cite{zhao2011robust}, computer vision area such as object tracking\cite{yuan2011discovering}\cite{chen2011detection}\cite{li2011graph}\cite{yang2011contour}, cluster analysis\cite{liu2010robust}, etc. Also, its low computation time and memory complexity make the realisty application feasible and attracted, especially in large-scale data size case. However, little theoretical work has been done to strengthen the solidness of the algorithm except for empirical demonstration, and to be honestly speaking, the correctness of the result always lays in doubt without theoretical guarantees.

The GS Algorithm originates from the the Dominant Sets and Pairwise Clustering(DSPC) Algorithm\cite{pavan2003new}\cite{pavan2007dominant}, which treated the dense subgraph discovery problem as a constrained optimization problem and gave a solid definition on the so-called "dominant set", i.e., dense subgraph. Further modifies the existing DSPC procedure, the GS Algorithm adds a neighborhood expansion procedure to reinforce the learning result. By iteratively employing the Replicator Dynamics and the new added "Neighborhood Expansion", the GS Algorithm claims to find the local maximum of the constraint objective function after finite number of iterations and further empirically demonstrates the claim.

However, to the best of our knowledge, none of the existing theoretical work has been done to certify the claim, nor does issues including the objective functions' behavior during the procedures, the stopping criteria. All of the above issues are closely related to one thing: the GS Algorithms' convergence property. That is to say, we need to ensure that the generated sequence set is convergent or at least contains a convergent subsequence set. It is certainly crucial to have a theoretical analysis about the GS Algorithm's convergence before we can confidently utilize it.

%Since an iterative procedure is employed in the GS Algorithm to find the local maximum of the constraint objective function, the convergence should be guaranteed before it widely applicates into reality. That is to say, we need to ensure that the generated sequence set is convergent or at least contains a convergent subsequence set. However, none of the existing work has proved the convergence of the GS Algorithm. It is certainly crucial to have a theoretical analysis about the GS Algorithm's convergence before we confidently utilize it.

Convergence theorem of algorithms has been a long-time discussion topic since decades ago in literature, including the ones with an iterative sequence set. Take the fuzzy $c$-means algorithm(FCM) for instance. The original strict proof is Provided by Bedzek \cite{bezdek1980convergence}\cite{hathaway1989relational}, who employed the Zangwill's theory \cite{zangwill1969nonlinear}\cite{gunawardana2005convergence} to establish the sequence's convergence property. Hoppner\cite{hoppner2003contribution} proved the convergence of the axis-parallel variant of the Gustafson-Kessel's algorithm\cite{gustafson1978fuzzy} by applying the Banach's classical contraction principle\cite{istratescu2002fixed}, which is the general case of FCM. Groll\cite{groll2005new} used the equivalence between the original and reduced FCM criteria, and conducted a new and more direct derivation of the convergence properties of FCM algorithms. Besides these, Selim\cite{selim1984k} treated the k-means clustering problem as a nonconvex mathematical program and provided a rigorous proof of the finite convergence of the K-means-type algorithms.

It may be intuitive that the FCM's convergence discussion could be applied to the GS Algorithm for both of them operating on an iterative set. However, it is not straightforward in implementation as the hardness of capturing GS Algorithm's complex characters. To address this problem, we provide a theoretical analysis of the algorithm. We start with the understanding of principal characteristics of the GS Algorithm by breaking it down into three key components, including generated sequence set, objective function and mapping, and then propose a framework to map such components to the conditions required in the Zangwill's theorem. We find that the mapped GS Algorithm can then perfectly match with the key requirements in the Zangwill's theorem. The convergence theorem for the GS Algorithm is then brought about.
% Comparing to this k-means typed algorithm's convergence proof, the most difficult point is the abstraction from detail algorithm.

Furthermore, a definition of the so-called "GS-type algorithm" is then given to provide us with a general view of algorithms with similar properties. More importantly, we build up a systematic learning on them by analyzing the objective functions' behaviors and observing their interesting resulting in the implemental results. We illustrate the proposed convergence theorem in terms of proving both the the GS Algorithm \cite{liu2010robust} and the DSPC Algorithm \cite{pavan2007dominant}, and confirm with the experimental results.

After all, our contributions here are listed as follows:
\begin{enumerate}
\item We theoretically analyze the convergence behavior of the GS Algorithm.
\item We have proven both the GS Algorithm and the DSPC Algorithm terminates a local maximum value, or at least contains a subsequence which converges to a local maximum.
\item A convergence proof framework is builded to make a better generalization of our work.
\end{enumerate}

The paper is organized as follows. Section \ref{sec_2} introduces the principle of the GS Algorithm and also a details description of it. In Section \ref{sec_3}, we first introduce the Zangwill's convergence theorem, and then extracts three key components in the GS Algorithm, mapping them to the Zangwill's properties. Section \ref{sec_4} discusses the convergence of the GS Algorithm and also analyzes some features of the algorithm. We extend the convergence proof to other GS-type algorithms and build up a framework in Section \ref{sec_5}. Experiments are conducted in Section \ref{sec_6} to verify the convergence theorem and behavior. Conclusions and future work can be found in Section \ref{sec_7}.

\section{Preliminaries} \label{sec_2}
\subsection{Rationale of the GS Algorithm}
The basic principle of the Graph Shift Algorithm is set forth in the work of \cite{liu2010robust}. In the perspective of graph mining, the GS Algorithm aims at searching each vertex's dense "nearer" subgraph with strong internal closeness. Two procedures of Replicator Dynamics and Neighborhood Expansion are recursively employed on each vertex sequentially to reach the goal. The former largely shrinks the identified subgraph, and the later expands the existing subgraph, both shift towards a local graph mode.
%%%%%%%%%%%%% here needs the revising

In \cite{liu2010robust}\cite{pavan2007dominant}, a probabilistic coordinate on Graph $G$ is defined as a mapping: $X:V\to\Delta^n$, where $\Delta^n=\{\boldsymbol{x}\in R^n:\boldsymbol{x}_i\ge0,i\in\{1,\cdots,n\}~\textrm{and}~|\boldsymbol{x}|_1 =1\}$, the support of $\boldsymbol{x}\in\Delta^n$ is the indices of all non-zero components, denoted as $\delta(\boldsymbol{x})=\{i|\boldsymbol{x}_i\neq0\}$, corresponding to a subgraph $G_{\delta(\boldsymbol{x})}$, and $\boldsymbol{x}_i$ denotes node $i$'s attendance in the subgraph $G_{\delta(\boldsymbol{x})}$ to some extent.
%one probabilistic cluster (also defined as subgraph) is defined as: $\boldsymbol{x}=\{\boldsymbol{x}\in R^n~|~\boldsymbol{x}_i\ge0,i\in\{1,\cdots,n\}~\textrm{and}~|\boldsymbol{x}|_1 =1\}$, and $\boldsymbol{x}_i$, $i$-th component of $\boldsymbol{x}$, denotes node $i$'s attendance to the specific subgraph $\boldsymbol{x}$.

The algorithm operates on an affinity matrix $A=(a_{ij})^{n\times n}$, in which $a_{ij}$ measures the similarity between node $i$ and node $j$. Then subgraph $G_{\delta(\boldsymbol{x})}$'s internal similarity is expressed as:
\begin{equation} \label{t_eq1}
g(\boldsymbol{x}):=a(\boldsymbol{x},\boldsymbol{x})=\sum_{i,j=1}^n a_{ij}\boldsymbol{x}_i \boldsymbol{x}_j = \boldsymbol{x}^T A\boldsymbol{x}.
\end{equation}

Accordingly, a local maximum solver of $g(\boldsymbol{x})$ can be taken to represent the desired dense subgraph. The identification of such local maximum regions is equivalent to solving the following quadratic optimization problem:
\begin{equation} \label{t_eq2}
\left\{\begin{array}{lc}
\textrm{maximize} & g(\boldsymbol{x})=\boldsymbol{x}^T A\boldsymbol{x}\\
\textrm{subject to} & \boldsymbol{x}\in\Delta^n
\end{array}\right.
\end{equation}

\subsection{Mapping definition}

%To further explore the properties of graph shift algorithm, we need some notations here to state the problems more formally.
%Some notations are made here to further discuss the properties of the GS Algorithm.
We begin the discussion with a formal definition on the Graph Shift Algorithm and its corresponding mapping. These clear predefined related concepts would facilitate much on the problem statement and understanding.

In general, the GS Algorithm defines a mapping $T_m:\Delta^n\to\Delta^n$ to get the iterative sequence set as:
\begin{equation}
\boldsymbol{x}^{(k)}=T_m(\boldsymbol{x}^{(k-1)})=\cdots=(T_m)^{(k)}(\boldsymbol{x}^{(0)}); k= 1,2,\cdots.
\end{equation}
Where $\boldsymbol{x}^{(0)}$ is an initial starting point, and superscripts in parentheses correspond to the iteration number. Therefore, the problem of this paper is whether or not the iterative sequence set $\{\boldsymbol{x}^{(k)}\}_{k=1}^{\infty}$ generated by $T_m$ converges to a local maximum solver of problem (Equation (\ref{t_eq2})).

Specifying the mapping $T_m$ more clearly will be both necessary and of great help in analyzing this question. Accordingly, $T_m$, the combination of Replicator Dynamics procedure $(B^{m_k})$ and Neighborhood Expansion procedure$(C)$, is broken down as:
\begin{equation} \label{t_eq7}
T_m:=B^{m_k}\circ C.
\end{equation}
 In Equation (\ref{t_eq7}), $B^{m_k}$ represents the $k$-th ($k\le m$) Replicator Dynamics procedure; $m_k$ corresponds to transformation $B$'s number in the $k$-th Replicator Dynamics procedure when a subgraph's mode is reached in this procedure (this result is actually a special case of Theorem 3); $B$ is the transformation expressed as:
 \begin{equation}
 \begin{split} \label{t_eq8}
B:&\Delta^n\to\Delta^n, \boldsymbol{x}(l_k)\to\boldsymbol{x}(l_k+1)\\
=&(\frac{\omega_1(l_k)\boldsymbol{x}_1(l_k)}{\sum_{i=1}^n \omega_{i}(l_k)\boldsymbol{x}_i(l_k)}, \cdots, \frac{\omega_n(l_k)\boldsymbol{x}_n(l_k)}{\sum_{i=1}^n \omega_{i}(l_k)\boldsymbol{x}_i(l_k)}).
\end{split}
\end{equation}
In Equation (\ref{t_eq8}), $\omega_i(l_k)=(A\boldsymbol{x}(l_k))_i=\sum_{j=1}^na_{ij}\boldsymbol{x}_j(l_k)$, %$\boldsymbol{x}(l_k)^TA\boldsymbol{x}(l_k)=\sum_{i,j} \boldsymbol{x}_i(l_k)\boldsymbol{x}_j(l_k)a_{ij}=\sum_{i=1}^n \omega_i(l_k)\boldsymbol{x}_i(l_k)$, and
$i\in\{1, \cdots, n\}$, $l_k\in\{1, \cdots,m_k-1\}$.
%Here $t$ stands for the whole iterations respectively,

$C$ is the Neighborhood Expansion procedure, denoted as:
\begin{equation} \label{t_eq10}
\boldsymbol{x}^{(k+1)} = \boldsymbol{x}^{(k)} + \Delta \boldsymbol{x} = \boldsymbol{x}^{(k  )} + t^*\boldsymbol{b}.
\end{equation}
Details of $t^*$ and $\boldsymbol{b}$ are explained in Appendix A.

%support of $\boldsymbol{x}$.% which is also named as the mode of the graph.
%%%%%%%%%%%%%%% revising here

%In summary, we call algorithms satisfied these three conditions as Graph-Shift liked algorithm.
%\textcolor{red}{Needs to be revised later. These three key components depict the major characteristics of the GS Algorithm, namely the generated sequence set $\{\boldsymbol{x}^{(k)}\}_{k=0}^\infty$, the objective function $g(\boldsymbol{x})=\boldsymbol{x}^TA\boldsymbol{x}$, and the mapping $T_m=B^{m_k}\circ C$. The GS Algorithm can then be described in terms of these three key components.}

\subsection{Detail Procedure and Stopping Criteria}
%In \cite{liu2010robust}, the graph shift algorithm is implemented by recursively using the Replicator Dynamics procedure and Neighborhood Expansion Procedure, with which these two will be shown later, until we find the desired solution.

The GS Algorithm\cite{liu2010robust} is an iterative process through the loop in seeking the dense subgraph starting from each vertex in the graph, with the pseudo-code shown below illustrating the whole process.
%\begin{algorithm}[!ht] \label{eq:algo1}
%\caption{Graph Shift Clustering\cite{liu2010robust}}
\begin{algorithmic}[1]
\REQUIRE $A^{n\times n}$, Affinity matrix of the whole data set with the diagonal value 0;\\
$\{\boldsymbol{x} = \{\boldsymbol{x}_i\}_{i=1^n}\},$ initial starting points, usually taken as $\{\boldsymbol{e} = \{\boldsymbol{e}_i\}_{i=1^n}\}$
\FOR {i=1,...n}
\STATE do Replicator Dynamics(Equation (\ref{t_eq8})) of $\boldsymbol{x}_{i}$
\IF {(result $\boldsymbol{x}$ is the mode of graph)}
    \STATE go to step 11
\ENDIF
\STATE do Neighborhood Expansion(Equation (\ref{t_eq10})) of $\boldsymbol{x}_{i}$
\IF {(result $\boldsymbol{x}$ is the mode of graph)}
    \STATE go to step 11
\ENDIF
\ENDFOR
\RETURN the belonging clusters $\boldsymbol{c}=\{\boldsymbol{c}_i\}_{i=1}^n$ corresponding to the starting points $\boldsymbol{x}=\{\boldsymbol{x}_i\}_{i=1}^n$.
\end{algorithmic}
%\end{algorithm}\\
The stopping criteria of the algorithm , or the solution set of the problem (Equation (\ref{t_eq2})) is set to satisfy the Karush-Kuhn-Tucker (KKT) condition \cite{programmin1951h}:
 %which is converted to the following expression after a series of transformations \cite{liu2010robust}:
 \begin{equation} \label{t_eq3}
\Gamma:=\{\boldsymbol{x}\in\Delta~|~\boldsymbol{x}~\textrm{satisfies} ~(A\boldsymbol{x})_i \left\{ \begin{array}{cc}
=\lambda, & i\in\sigma(\boldsymbol{x});\\
\le\lambda, & i \notin \sigma(\boldsymbol{x}).
\end{array} \right.\}.
 \end{equation}
Here $(A\boldsymbol{x})_i$ is the $i$-th component of $A\boldsymbol{x}$; $\lambda$ is one Lagrange multiplier. $\sigma(\boldsymbol{x})=\{i\in\{1, \cdots, n\}|\boldsymbol{x}_i\neq 0\}$ corresponds to the subgraph as defined above.

As stated above, the GS Algorithm is implemented on each vertex's evolving process by recursively using the Replicator Dynamics procedure and Neighborhood Expansion Procedure, until it reaches the desired solution, i.e., satisfying KKT condition (Equation (\ref{t_eq3})) to each of the starting vertices.

%\textcolor{red}{GS starts with initial points as each vertices, noted as $\{\boldsymbol{e}_i, i=1,\cdots, n\}$, and then repeats estimating the maximum in the subgraph and estimation the shift from one dense subgraph to a nearer dense subgraph until it converges.}

%\textcolor{red}{The number of iterations depends on the dimension of the similarity matrix. Since the Replicator Dynamics has maintained most of the computation in estimating the dense subgraph, the algorithm's computation costs a lot. Complexity estimation and about the computation load.}

\section{Mapping from the GS Algorithm to the Zangwill's Theorem} \label{sec_3}

% Compared to all the convergence proves in the iterated sets scenario,
The Zangwill's convergence theorem \cite{zangwill1969nonlinear}\cite{gunawardana2005convergence} is fundamental in terms of proving the convergence of iterative sets for its general applicability. In this section, we build up the mapping from the GS principle to the Zangwill's convergence theorem.
%%%%%%%%%%%%%%%%%%here needs adding some contents

\subsection{Zangwill's Convergence Theorem}
Definitions and lemmas are introduced before we present the Zangwill's convergence theorem.

\begin{definition}
A point-to-set mapping $\Omega$ from set $X$ to power set $Y$ is defined as $\Omega: X\to P(Y)$, which associates a subset of $Y$ with each point in $X$,  $P(Y)$ denotes the power set of $Y$.
\end{definition}

\begin{definition}
Given a function $f$ and an element $c$ of the domain $I$, $f$ is said to be continuous at the point $c$ if the following holds: for every $\varepsilon>0$, there exists a $\eta > 0$ such that for all $x\in I$, $|x-c|<\eta\Rightarrow|f(x)-f(c)|<\varepsilon$.
\end{definition}

\begin{definition}
A point-to-set mapping $\Omega:X\to P(Y)$ is said to be closed at a point $\boldsymbol{x}^*$ in $X$ if $\{\boldsymbol{x}^{(m)}\}\subset X$ and $\boldsymbol{x}^{(m)}\to\boldsymbol{x}^*, \boldsymbol{y}^{(m)}\in\Omega(\boldsymbol{x}^{(m)})$ and $\boldsymbol{y}^{(m)}\to\boldsymbol{y}^*$ imply that $\boldsymbol{y}^*\in\Omega(\boldsymbol{x}^*)$.
\end{definition}

The following Lemma \ref{eq:lemma1} is induced by integrating a continuous function with a point-to-set mapping:
\begin{lemma} \label{eq:lemma1}
Let $C:M\to V$ be a function and $B:V\to P(V)$ be a point-to-set mapping. Assume $C$ is continuous at $\omega^*$ and $B$ is closed at $C(\omega^*)$, then the point-to-set mapping $A=B\circ C:M\to P(V)$ is closed at $\omega^*$.
\end{lemma}

The composition of continuous functions is still a continuous function, we have Lemma \ref{eq:lemma2}:
\begin{lemma} \label{eq:lemma2}
Given two continuous functions: $f:I \to J(\subset\boldsymbol{r}), g:J\to\boldsymbol{R}$, the composition $g\circ f:I\to\boldsymbol{R},x\mapsto g(f(x))$ is continuous.
\end{lemma}

Accordingly, the Zangwill's convergence theorem is described below.
\begin{theorem}
Given an algorithm on $X$, $\boldsymbol{x}^{(0)}\in X$, assume the sequence $\{\boldsymbol{x}^{(k)}\}_{k=1}^{\infty}$ is generated which satisfies
\begin{equation}
\boldsymbol{x}^{(k+1)}\in \boldsymbol {A}(\boldsymbol{x}^{(k)})
\end{equation}
For a given solution set $\Gamma\subset X$ of an algorithm, if the following three properties holds:
\flushleft
\begin{description}
\item[Compact] The sequence set $\{\boldsymbol{x}^{(k)}\}_{k=0}^{\infty}\subset S$ for $S\subset X$ is a compact set.
\item[Decreasing] There is a continuous function $Z$ on $X$ such that
\begin{enumerate}
\item[1)] if $\boldsymbol{x}\notin\Gamma$, then $Z(\boldsymbol{y})<Z(\boldsymbol{x})$ for all $\boldsymbol{y}\in\boldsymbol{A}(\boldsymbol{x})$.
\item[2)] if $\boldsymbol{x}\in\Gamma$, then $Z(\boldsymbol{y})\leq Z(\boldsymbol{x})$ for all $\boldsymbol{y}\in\boldsymbol{A}(\boldsymbol{x})$.
\end{enumerate}
\item[Closed] The mapping $\boldsymbol{A}$ is closed at all points of $X \backslash \Gamma$.
\end{description}

Then either the algorithm stops at the point where a solution is identified or there exists such a $k$ so that for all $k+j$ ($j \ge 1$) there is a convergent subsequence of $\{\boldsymbol{x}^{(i_k)}\}_{k=0}^{\infty}$ in the solution set $\Gamma$.
\end{theorem}

The Zangwill's convergence theorem provide a feasible direction to verify one algorithm's convergence behavior, especially ones with iterative implementations. With its general flexibility, it has been applied widely to prove the convergence of algorithms with similar properties, including clustering and optimization research. Amongst all the iterative algorithms, here we are interested, in particular, in monotonic algorithms.

\subsection{Mapping}

Retrospective to the conditions in Zangwill's theorem, we further break down the GS Algorithm and abstract the following characteristics from it (detailed verification will be given later):
\begin{description}
\item[1), Simplex of generated sequence set] The candidate solution sequence set $\{\boldsymbol{x}^{(k)}\}_{k=0}^\infty$ generated by the mapping $T_m$ lies in $\Delta^n=\{\boldsymbol{x}\in R^n:\boldsymbol{x}_i\ge0$ and $|\boldsymbol{x}|_1 =1\}$, i.e., the standard $n$-simplex of $R^n$ in any step $k$ ($k\le m$);
\item[2), Monotonic and continuous objective function] The objective function $g(\boldsymbol{x})=\boldsymbol{x}^TA\boldsymbol{x}$ is continuous and strictly increases during the mapping $T_m$ according to the Propositions \ref{t_prop2}-\ref{t_prop4} (see Section 4);
\item[3), Closed mapping] The mapping $T_m=B^{m_k}\circ C$ is closed during each procedure in accordance to Propositions \ref{t_prop5}-\ref{t_prop6} (see Section 4) at all points of the generated sequence set.
\end{description}

%\textcolor{red}{Three formal requirements were defined in the Zangwill's convergence theorem to validate the convergence from this perspective, sharing the similarities with the three GS Algorithm's properties as discussed in Section 2.}
%The Zangwill's convergence theorem defines formal requirements on the generated sequence set $\{\boldsymbol{x}^{(k)}\}_{k=0}^\infty$, the objective function $g(\boldsymbol{x})$ and the mapping $\boldsymbol{A}$ for an algorithm that converges.
%These three properties share similarity with the three GS Algorithm's properties as discussed in Section 2, which also focus on the generated sets, objective function and mapping.

These three properties are also the key components in one algorithm, that is to say, with these vital feature requirements clearly prescribed, the algorithm is fixed into a predefined framework, including the generated sequence set defining the scope of the variables, the objective function's behavior describing the algorithm mapping's efforts towards the setting goal and mapping itself with restricted property.

%\textcolor{red}{needs revising.} Based on the definition of standard $n$-simplex of $R^n$, the GS Algorithm's generated set is bounded. If the set is also closed, then the generated set is compact too. Converting the objective function $\hat{g}(\boldsymbol{x})$ to $\pm g(\boldsymbol{x})$ can adjust the monotonicity of the objective function to be decreasing. Therefore, a perfect match can be built between the GS Algorithm and the Zangwill's theorem.

Table \ref{ttab1}. depicts the one-to-one correspondence similarities between these two more clearly.
\begin{table}[!ht] %\scriptsize \small
\caption{Mapping between GS Algorithm and Zangwill's Theorem}
\label{ttab1}\begin{center}
\begin{tabular}{@{}p{32mm}ccp{64mm}c@{}}\toprule
%\multicolumn{2}{c}{Item} \\ \cmidrule(r){1-3}
the GS Algorithm &   & the Zangwill's theorem        \\ \midrule
${Simplex}$ & $\thicksim$ & ${Compact}$ \\
${Monotonic}$  & $\thicksim$  & ${Decreasing}$\\
${Closed}$ &  $\thicksim$  &  ${Closed}$  \\  \bottomrule
\end{tabular}
\end{center}
\end{table}
%The discussion below can help us make a deeper understanding of this relationship.

\section{Convergence of the GS Algorithm} \label{sec_4}
%Now we present some propositions, all to clarify the relationships between Graph Shift algorithm and Zangwill's Convergence theorem.
Several propositions are declared before analyzing deeply into the convergence behavior of the GS Algorithm, all focus on the three properties as we discussed in Section 2.
%Before we discuss the convergence, we propose several propositions for three GS Algorithm's components: generated sequence set, objective function and mapping.
Detailed proofs of these propositions are given in Appendix B-E.
%For $\bf{Stable}$, we can prove it equals to Zangwill's theorem's $\boldsymbol{Z1}$.
%For $\{\boldsymbol{x}_{(k)}\}_{k=0}^{\infty}$, we can find its lying space is rather explorable.

The GS Algorithm's stable solution set is compact according to Proposition \ref{t_prop1}.
\begin{proposition} \label{t_prop1}
The sequence set $\{\boldsymbol{x}^{(k)}\}_{k=0}^{\infty}\subset S$ generated by the mapping $T_m = B^{m_k}\circ C$ is a compact set.
\end{proposition}

Propositions \ref{t_prop2}-\ref{t_prop4} discuss the monotonicity of $g(\boldsymbol{x})$ under the mapping $T_m = B^{m_k}\circ C$, as discussed about the monotonicity of the GS Algorithm's main characteristics.
\begin{proposition} \label{t_prop2}
The objective function $f(\boldsymbol{x})=\boldsymbol{x}^T\boldsymbol{A}\boldsymbol{x}$ strictly increases along any nonconstant trajectory of Equation (\ref{t_eq8}) when $\boldsymbol{x}\in X/\Gamma$.
\end{proposition}

\begin{proposition} \label{t_prop3}
The objective function $f(\boldsymbol{x})=\boldsymbol{x}^T\boldsymbol{A}\boldsymbol{x}$ strictly increases along the neighborhood expansion operation of Equation (\ref{t_eq10}).
\end{proposition}
\begin{proof}
It can be derived from the definition of $\Delta\boldsymbol{x}$ in Appendix B.
\end{proof}

\begin{proposition} \label{t_prop4}
$g(x)=\boldsymbol{x}^TA\boldsymbol{x}$ is a function both continuous and strictly increasing during the mapping $T_m=B^{m_k}\circ C$ when $\boldsymbol{x}\in X/\Gamma$, but just increasing if $\boldsymbol{x}\in \Gamma$.
\end{proposition}

%We can also find the mapping $T_m$'s evaluation is closely related to the value of $\boldsymbol{x}$.
Propositions \ref{t_prop5}-\ref{t_prop6} validate the closed mapping property of the GS Algorithm.
\begin{proposition} \label{t_prop5}
The mapping $C$ is closed on all points of $X\backslash \Gamma$.
\end{proposition}

\begin{proposition} \label{t_prop6}
The mapping $T_m=B^{m_k}\circ C$ is closed on $X/\Gamma$.
\end{proposition}
\begin{proof}
According to the definition of $B$ (Equation (\ref{t_eq8})), it is continuous on $X/\Gamma$. C is closed on $X/\Gamma$ as per Proposition \ref{t_prop5}. According to Lemma \ref{eq:lemma1}, $T_m$ is closed on $X/\Gamma$.
\end{proof}

With the above preparations, we have the following Theorem 2.
\begin{theorem}
Let $A = (a_{ij})^{(n\times n)}$ be a similarity matrix with diagonal values 0, $T_m$, $\Gamma$ be defined as Equation $(\ref{t_eq7})$, Equation $(\ref{t_eq3})$, and $\boldsymbol{x}^{(0)}$ be an arbitrary initial starting point, then either the iteration sequence $\{\boldsymbol{x}^{(r)}\}$ ($r = 1, 2, \ldots$)
terminates at a point $\boldsymbol{x}^*$ in the solution set $\Gamma$ or there is a subsequence converging to a point in $\Gamma$.
%The sequence set $\{\boldsymbol{x}^{(k)}\}_{k=0}^\infty$ generated by the mapping $T_m=B\circ C$ will be convergent.
\end{theorem}

\begin{proof}
Taking $\hat{g}(\boldsymbol{x})=-g(\boldsymbol{x})=-\boldsymbol{x}^TA\boldsymbol{x}$ as the continuous function $Z$ and $T_m$ as the algorithm mapping $A$ in Theorem 1. Proposition 1 shows that the sequence set $\{\boldsymbol{x}^{(k)}\}_{k=1}^{\infty}$ generated by $T_m$ is a compact set. $\hat{g}(\boldsymbol{x})$ is continuous and strict decreasing as the continuity and strict increasing characteristic of $g(\boldsymbol{x})$ in the trajectory of $T_m$ are proven by Proposition \ref{t_prop4}. Proposition \ref{t_prop6} asserts $T_m$ is closed on $x/\Gamma$ ($\Gamma$ is the solution set defined in Equation (\ref{t_eq3})). According to the Zangwill's convergence theory, Theorem 2 holds as all three properties are satisfied.
\end{proof}

This result gives us the theoretical guarantees to the various applications of the GS Algorithm and ensure to reach at least a local maximum of the objective function after finite number of algorithm's mapping implementations.

\section{Convergence for Other GS-type Algorithms} \label{sec_5}
%We call algorithms satisfied three key conditions stable, monotonic, and closed as the GS-typed algorithm.
We expand the proposed GS convergence theorem's proof to other GS-type algorithms. Take the Dominant Sets and Pairwise Clustering (DSPC) Algorithm \cite{pavan2007dominant} as an example, for it is the origin method of the GS Algorithm. The DSPC Algorithm shares the same goal as the GS Algorithm in terms of finding dense subgraphs, i.e., dominant sets $\sigma(\boldsymbol{x})$. Their implementations are mostly similar, however differentiates in whether selecting the neighborhood expansion or not, the GS Algorithm does but the DSPC Algorithm does not.

%  Along with these two algorithms' similarities, properties on the key aspects, including generated sets, objective function and mapping are both possessed on them.
The detail implementation of the DSPC Algorithm could refer to \cite{pavan2007dominant}, due to the simiplicy, we ignore it here and focus on its convergence behavior. The DSPC Algorithm consists of three key components, holding similar properties to the GS Algorithm, however, differing in minor places.
\begin{description}
\item[1), Simplex of generated sequence set] The sequence set $\{\boldsymbol{x}^{(k)}\}_{k=0}^\infty$ generated by mapping $B$ always lies in $\Delta^n$;
\item[2), Monotonic and continuous objective function] The objective function $g(\boldsymbol{x})=\boldsymbol{x}^TA\boldsymbol{x}$ is continuous and  strictly increasing during the mapping $B$;
\item[3), Continuous mapping] The mapping $B$ is continuous.
\end{description}
Table \ref{ttab2} further displays the relationship between the DSPC Algorithm, the GS Algorithm and the Zangwill's theorem.
\begin{table}[htbp] %\scriptsize \small
\caption{\small{GS-type algorithms and Zangwill's theorem's mapping}}
\label{ttab2}
\begin{center}
\begin{tabular}{@{}ccccc@{}}\toprule
DSPC Algorithm & & GS Algorithm &   & Zangwill's theorem        \\ \midrule
$Simplex$ & $\thicksim$& $ Simplex$ & $\thicksim$ & $Compact$     \\
${Monotonic}$ & $\thicksim$& $ {Monotonic}$  & $\thicksim$  & $Decreasing$   \\
${Continuous}$ & $\thicksim$& $ {Closed}$ &  $\thicksim$  &  $Closed$  \\  \bottomrule
\end{tabular}
\end{center}
\end{table}

A convergence proof has been provided for the continuous version of the DSPC Algorithm in \cite{weibull1997evolutionary}. Here some related propositions are first introduced. Then we discuss the convergence property of its discrete-time version by involving the Zangwill's convergence theorem.

\begin{proposition} \label{t_prop8}
The sequence set $\{\boldsymbol{x}^{(k)}\}_{k=0}^{\infty}\subset S$ generated by the mapping $B^{m_k}$ is a compact set.
\end{proposition}
For its proof, refer to Appendix B.
%The proof of Proposition \ref{t_prop1}.

\begin{proposition} \label{prop7}
If a mapping $f:S\to T$ is continuous on $S$, then $f$ is closed on $S$.
\end{proposition}

\begin{theorem}
Let $A = (a_{ij})^{(n\times n)}$ be a similarity matrix with diagonal values 0, $B$ be defined as Equation $(\ref{t_eq8})$, and
$\boldsymbol{x}^{(0)}$ be an arbitrary initial starting point, then either the iteration sequence $\{\boldsymbol{x}^{(r+1)}=B(\boldsymbol{x}^{(r)})\}, (r = 1, 2, \ldots)$  terminates at a point $\boldsymbol{x}^*$ in the solution set $\Gamma$ or there is a subsequence converging to a point in $\Gamma$.
\end{theorem}

\begin{proof}
Taking $\hat{g}(\boldsymbol{x})=-g(\boldsymbol{x})=-\boldsymbol{x}^TA\boldsymbol{x}$ as the continuous function $Z$ and $B$ as the algorithm mapping $A$ in Theorem 1. Proposition \ref{t_prop8} shows that the sequence set $\{\boldsymbol{x}^{(k)}\}_{k=1}^{\infty}$ generated by $B$ is a compact set. $\hat{g}(\boldsymbol{x})$ is continuous and strict decreasing according to the continuity and strict increase of $g(\boldsymbol{x})$ in the trajectory of $T_m$ are proven in Proposition \ref{t_prop2}. Proposition \ref{prop7} asserts $B$ is closed on $x/\Gamma$, while $\Gamma$ is the solution set defined in in Equation (\ref{t_eq3}). According to the Zangwill's convergence theory, Theorem 3 holds.
\end{proof}

%Here we further propose a framework with convergence properties for the convergence proof of GS-type algorithms.
There are many similarities shared between these two algorithms' proving process, along with their similar properties. Thus, we can propose a Zangwill's theorem-based convergence framework for the "similar algorithms". Firstly, the so-called "GS-type algorithm" is defined referring to what we called "similar algorithms".
\begin{definition}
An algorithm is a GS-type algorithm if and only if it satisfies the conditions on three key components: simplex of generated sequence set, monotonic and continuous objective function and closed mapping.
\end{definition}
%We define algorithms satisfied the above three key conditions as GS-typed algorithm.
% As shown above, the generated sequence set  by a GS-type algorithm is convergent or contains a convergent subsequence set.

We provide a flowchart (Fig. \ref{fig_f}) to illustrate our prooving process in details. It presents a guideline proving the convergence of GS-type algorithms step by step. (1) Break down the GS-type algorithm into three parts: generated set, objective function and mapping. (2) Check if these three parts all satisfy the corresponding requirements. Any part unsatisfied is regarded as unsuitable for applying this framework, otherwise it is convergent.

\begin{figure}[!ht]
\caption{Proving framework for GS-type algorithm convergence}
\centering
\includegraphics[scale=0.30, width = 0.3 \textwidth, bb = 65 350 503 727, clip]{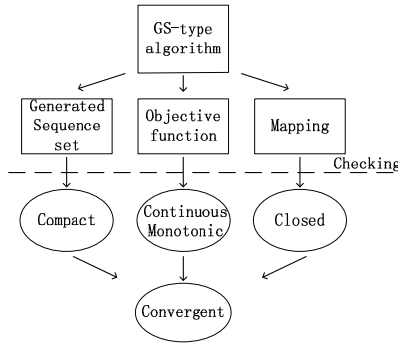}
\label{fig_f}
\end{figure}

Among the two processes in Fig. \ref{fig_f}, properties verification is usually the difficult part, especially the objective function's monotonic behavior. We will further discuss its behavior in the Section \ref{sec_6}.

\section{Experimental Verification} \label{sec_6}

% In this section, experimental results that illustrate the convergence properties of GSA are presented in various scenarios. For experimental results that illustrate performance advantages of GSA or DSPC algorithm, readers are referred to \cite{pavan2007dominant}\cite{liu2010robust}.

% Both GS algorithm and DSPC algorithm shows an amazing performance in application, especially image processing \cite{pavan2003new},\cite{chenmulti},\cite{zhao2011robust}. According to the goal of this paper, its convergence rate is focused.

The GS Algorithm and the DSPC Algorithm are all implemented in MATLAB2011b. Our experiments are conducted on an Acer Aspire 4720Z laptop having an Intel Pentium DualCoreT2330(1.6GHz, 533MHz FSB, 1MB L2 cache), with 2GB DDR2 RAM, using LINUX operating system.
\subsection{Experimental Settings}
Since GS-type algorithms manipulate data based on a similarity matrix,
% rather than a vectorial representation of data points, its conducted data is designed to be one similarity matrix assuming to be abstracted from the real data set,
we construct a similarity matrix instead of real data sets with its element values uniformly sampling within the interval $[0,1]$, and the matrix dimensionality scaling from $100$ to $3000$. Also, we consider cases in which the matrices are fully dense matrices (FDM), partially dense matrices (PDM), and block tridiagonal matrices (BTM).

% The convergence rate, including iteration time, running time is tested according to our goal.
%To simplify the algorithm, we set a threshold $\eta = 0.0001$ so that if $\boldsymbol{x}_i^{(m)} < \eta$ then $\boldsymbol{x}_i^{(m)}=0$.

Initial starting point $\boldsymbol{x}$ can be randomly chosen or be the single vertice $\{I_i, i= 1,\cdots, n.\}$. In our experiments, we use the single vertice with the same as \cite{liu2010common}. We test the GS Algorithm \cite{liu2010robust} as well as the DSPC Algorithm \cite{pavan2007dominant}. Each algorithm runs for three times to obtain an averaged performance.
We verify the proposed the GS Algorithm convergence theory through experiments and focus on testing convergence performance. The number of transformations ($m_k$) in Replicator Dynamics, the whole iteration number ($m$), running time ($T$), and average iteration running time ($t=T/m$) are presented to evaluate the convergence performance.

\begin{table*}[!htb] \scriptsize \small
\caption{\small{the GS Algorithm and the DSPC Algorithm's testing results}}
\centering
\begin{threeparttable}
\begin{tabular}{@{}cllllll|cllll@{}}\toprule
\multicolumn{7}{c|}{the GS Algorithm} & \multicolumn{5}{c}{the DSPC Algorithm} \\
\cmidrule(r){1-12}
 Case & Scale       & $m_k$    &  $m$    & $T(s)$ & $t(s)$ & S.R.(\%)\tnote{1} & Case & Scale   & $m_k$        & $T(s)$  & S.R.(\%)\tnote{1}  \\ \midrule
 FDM& 100         & 1228.5   & 2.52    &       0.12    &   0.05 & 100.0&FDM& 100         & 953.2     &     0.14   &  100.0   \\
 &500         & 1297.3   & 2.96     &         0.40      &   0.14 &100.0& & 500         & 1298.5    &     0.37  &   100.0   \\
 &1000        & 1101.7   & 3.75     &         1.01      &   0.27 &100.0& &  1000        & 1604.2    &    1.10  &   100.0   \\
 &1500        & 1401.2   &    3.24 &        3.87       &   1.19 &100.0 & & 1500        & 1469.6    &    4.25  &   100.0    \\
 &2000        & 1575.4   &    3.49  &        20.00      &   5.73 &100.0& & 2000        & 1448.1    &     18.61 &   100.0   \\
 &3000        & 1511.4   &    3.56  &         50.15     &   14.09 &100.0& & 3000        & 1785.3    &    56.20 &   100.0    \\\midrule
 PDM& 100         & 236.3    & 2.22      &   0.02 & 0.01 & 26.1  & PDM    &  100        &    203.3   &  0.02   & 26.7\\
 &500         & 285.9    & 2.24      &   0.08  & 0.04 & 26.2      &        &   500       &     299.6  &   0.07 & 26.3\\
 &1000        & 275.4    & 2.38     &   0.20 & 0.09 & 26.4      &   &   1000       &    326.7    & 0.15  & 26.4\\
 &1500        & 275.9    &     2.40   &     0.94 & 0.39 & 26.4   &       &   1500       &    338.2  &   0.55  & 26.3\\
 &2000        & 348.1    &    2.41   &     3.32 & 1.38 & 26.3   &      &   2000       &    288.8  &   2.11  & 26.4\\
 &3000        & 332.9    &    2.50   &   7.75 & 3.10 & 26.3     &        &   3000       &    389.3   &  6.85  & 26.3\\\midrule
BTM&100         & 794.9     & 2.16     &      0.06  &     0.03 &  22.0 & BTM    &  100        &  717.6     &  0.05   & 22.0\\
&500         & 1185.3    & 2.65       &      0.31  &    0.12   & 22.0&       &   500       &  1177.7     &  0.28  & 22.0\\
&1000        & 1221.6    & 2.67      &    0.68     &   0.25     &22.0&   &   1000       &  1130.8  & 0.74  & 22.0\\
&1500        & 1198.6    &   2.97    &    2.96     &   1.00        &22.0&      &   1500       & 1226.5   & 2.63   & 22.0\\
&2000        & 1289.9    &    3.02    &      13.57  &   4.50       &22.1 &     &   2000       & 1417.8  & 13.24   & 22.0\\
&3000        & 1406.2    &   3.13    &     38.55 &    12.34     &22.0&        &   3000       & 1517.4  & 38.16   & 22.0\\\bottomrule
\end{tabular}
      \begin{tablenotes}
        \footnotesize
        \item[1] \tiny{S.R. refers to the sparse rate, the proportion of the number of non-zero elements to the number of whole elements.}
      \end{tablenotes}
    \end{threeparttable}
\label{ttab3}
\end{table*}

\subsection{Objective function behavior}
\begin{figure*}[ht]
    \begin{tabular}{cc}
    \begin{minipage}[t]{2.1in}
    \includegraphics[width=2.1in]{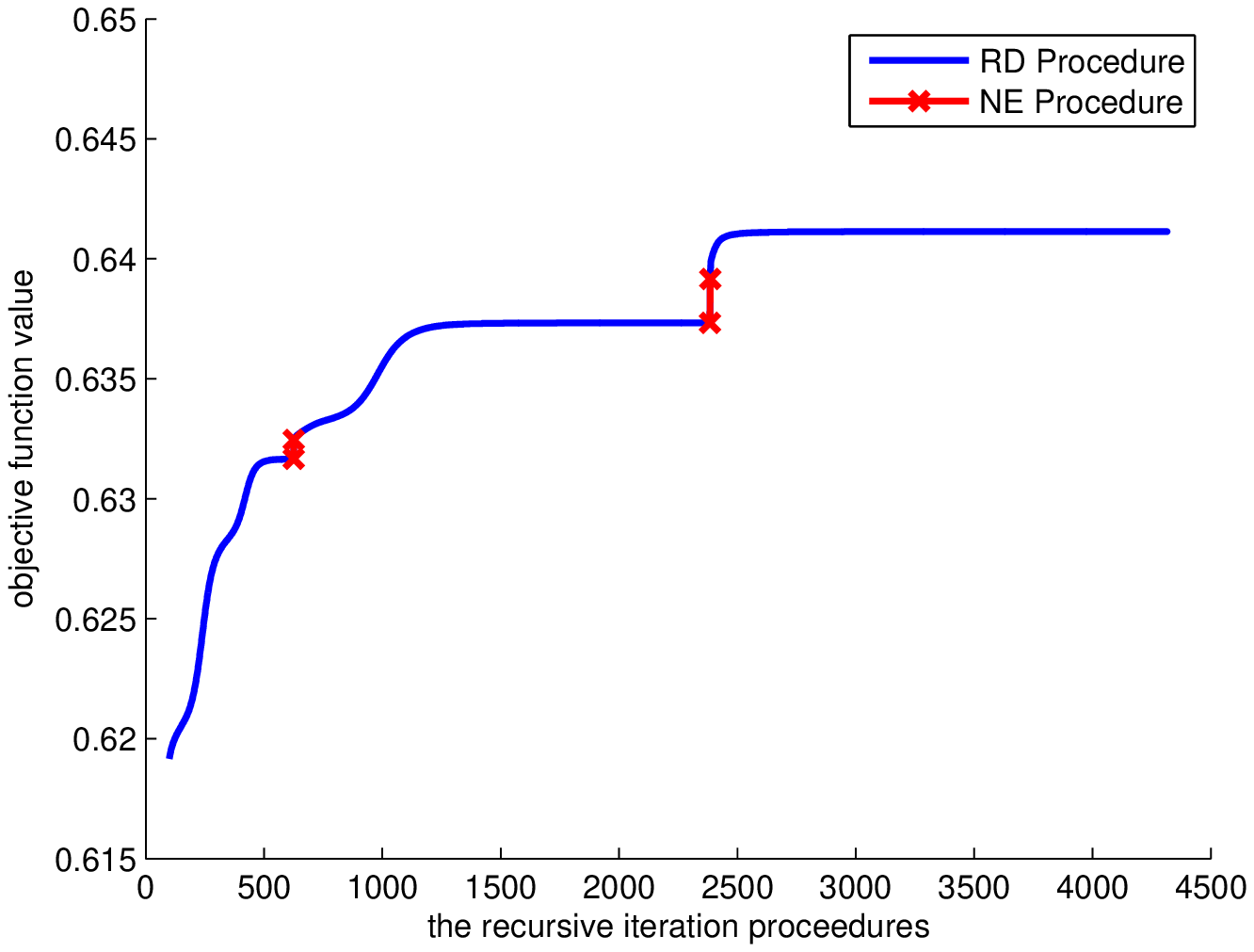}
    %\caption{Ve}
    \end{minipage}
    \begin{minipage}[t]{2.1in}
    \includegraphics[width=2.1in]{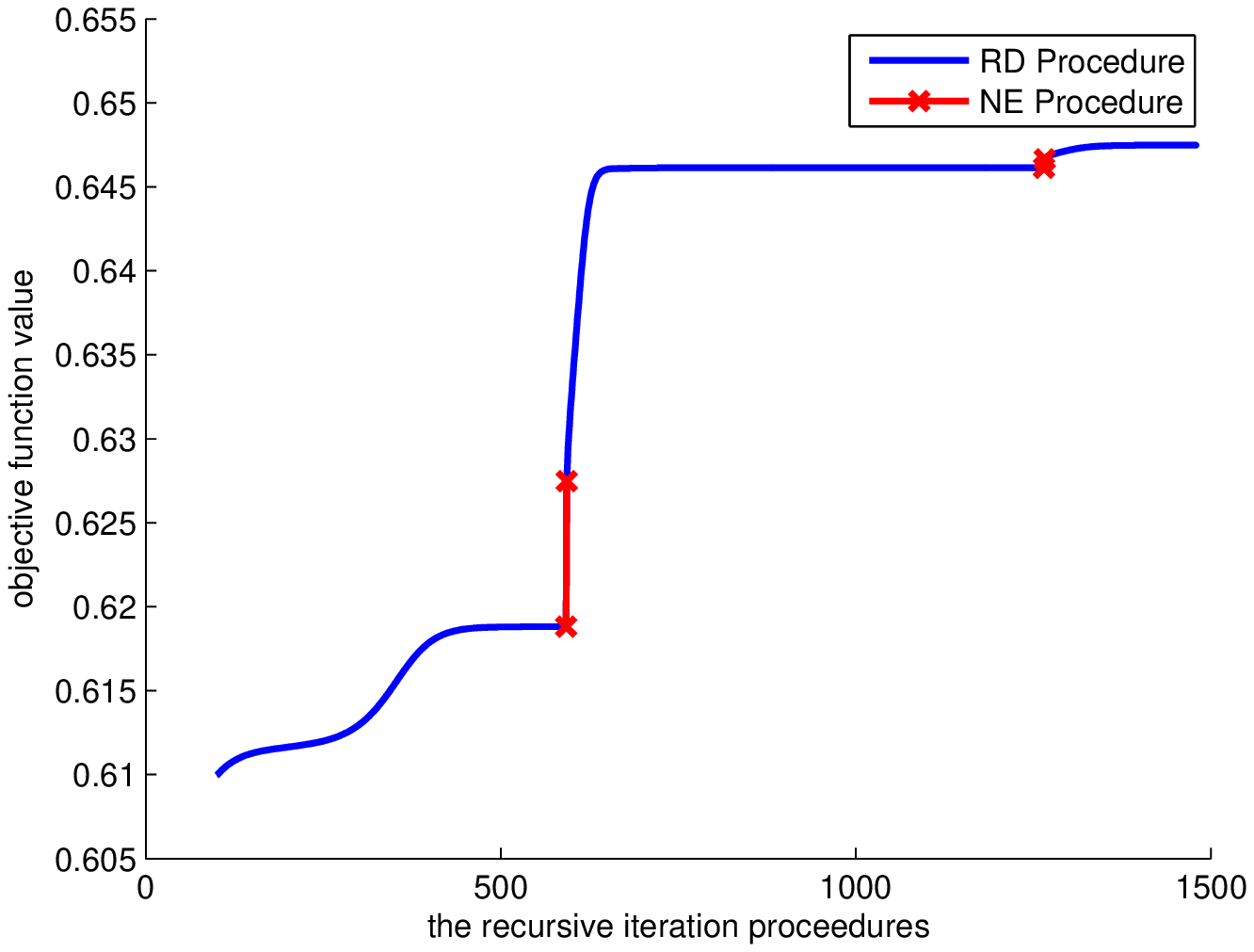}
    %\caption{This is a caption}
    \end{minipage}

    \begin{minipage}[t]{2.1in}
    \includegraphics[width=2.1in]{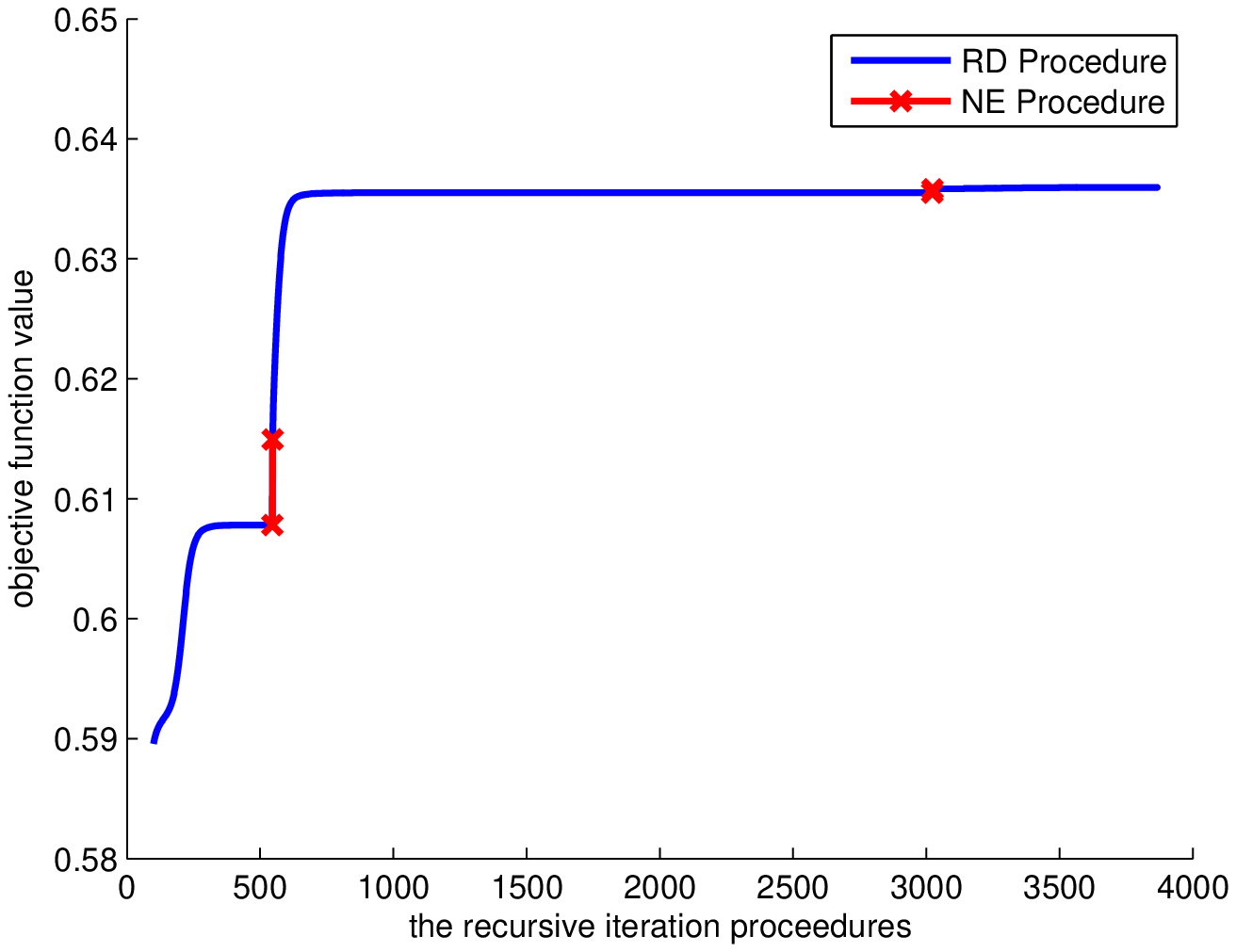}
    %\caption{This is a caption}
    \end{minipage}
\end{tabular}
\caption{Vertex 71(left), 231(middle) and 265(right)'s objective functions' behaviors(For better view, please see color pdf.)}
\label{fig_b}
\end{figure*}

Fig. \ref{fig_b} depicts the behaviors of three representative vertices' corresponding objective function in the GS Algorithm under the PDM case and 500 scale, along with the candidate solution evolving. The $X$-axis represents the evolving process times , and the $Y$-axis stands for the objective function's value. To be fairly compare the performance, we take $B$ and $C$'s time as equal. Thus, the rate of evolving time is $m_k:1$ between Replicator Dynamics and Neighborhood Expansion in the $k$-th iteration, according to Equation (\ref{t_eq7}).

From these three vertices' evolving behaviors, we can see all the three vertices reach a local maximum value. The objective function's value kept increasing with the evolving process, which is in accordance with the Proposition \ref{t_prop4}. This is perfectly matched with the theorem we have proposed.

When the curve first starting from the Replicator Dynamics, it always produces a steep increasing trends in the first few steps and then turn to a gentle curve in the last. What is more, when encounters the Neighborhood Expansion procedure, it always produce a huge jump comparing to the flat curve in the Replicator Dynamics. This is because the Neighborhood Expansion procedure is one that pulling the candidate solution from one local dense subgraph towards the dense graph and the dense value will change much with the subgraph changes.

Also, we can notice that most of the evolving time is the Replicator Dynamics, denoting that most of the calculations were done in Replicator Dynamics, i.e., searching the dense subgraph.

\subsection{Convergence performance}
The convergence performance of the GS Algorithm and the DSPC Algorithm depends on many factors: scale of the affinity matrix, matrix structure, element value, etc. We test different scenarios by changing the scale of matrix and the sparseness of matrix. The scenarios and results are given in Table \ref{ttab3}.

The experimental results presented in the GS Algorithm and the DSPC Algorithm are a little different for their different mapping definitions, with the former being $T_m = B^{m_k}\circ C$ and the later being $B$.

The results show that both the GS Algorithm and the DSPC Algorithm converge under the cases with FDM, PDM and BTM. In each case, transformation $B$'s number $m_k$ and iteration $T_m$'s number $m$ increase slowly when matrix's scaling grows, sometimes even with a small drop, e.g., in FDM case of the GS Algorithm, $m_k$'s values in the third row and $m$'s value in the fourth row are both smaller than the previous ones although matrices' scale increases. What is more, when the matrice's scale is 3000, which is 30 times of the first one's scale in each case, its $m_k$ and $m$'s values are less than 2 times larger. This indicates that if the dense rate of the matrix is fixed, the computational cost is always under control even though the matrices' scale increases.

 Also, the GS Algorithm's transformation number, iteration number and running time all decrease when the matrix becomes sparser, the same with the DSPC Algorithm's transformation number and running time. This indicates that if we incorporate prior information and setting the unrelated nodes' similarity value as 0, then we can save much computational cost in the application. However, both of the two Algorithms performs worse on the block tridiagonal matrices case compared to the partially dense matrices case, even with a smaller sparse rate. This is due to the fact that block tridiagonal matrices' calculation is quite similar to the one of smaller scale with full dense matrices. As shown above, it is heavily computational loaded.

% The results show that the whole iteration time increases as the scale become larger, and also the performance will be better if the matrix become sparse. Among these and the most important, it is convergent.

\section{Conclusions and Future work} \label{sec_7}
Graph Shift (GS) Algorithm shows great advantage on efficiently dealing with noisy data, however no theoretical outcome has been reported about its algorithm behavior and what is more, its convergence proof. In this paper, we have proposed a generic theoretical framework to prove the convergence of GS-type algorithms. A GS-type Algorithm consists of three key components: simplex of generated sequence set, monotonic and continuous objective function and closed mapping. They are mapped to the Zangwill's convergence theorem's three key conditions. Consequently, the convergence of the GS Algorithm is proved by applying the Zangwill convergence theorem.
% this paper asserted that GSA's generated sequence set is convergence, or at least, contains a convergent subsequence set.

We have shown that the framework can be applied to GS-type algorithms, as well as the Dominant set and pairwise clustering (DSPC) Algorithm. Experimental results on both the GS Algorithm and the DSPC Algorithm certified that they both converge under different scenarios in terms of the scale of the affinity matrix, the transformation number, and the iteration number.

However, this paper limits the generated sequence set under the simplex case, more work will be done to expand it to other compact set. What is more, Banach's contraction theory and optimization methods on nonconvex mathematical program are being considered in our problem so as to suit in a more general case.

%\section{Acknowledgments}

\bibliographystyle{abbrv}
\bibliography{convergence}  % sigproc.bib is the name of the Bibliography in this case

\begin{thebibliography}{10}

\bibitem{baum1967inequality}
L.~Baum and J.~Eagon.
\newblock An inequality with applications to statistical estimation for
  probabilistic functions of markov processes and to a model for ecology.
\newblock {\em Bull. Amer. Math. Soc}, 73(3):360--363, 1967.

\bibitem{bezdek1980convergence}
J.~Bezdek.
\newblock A convergence theorem for the fuzzy isodata clustering algorithms.
\newblock {\em Pattern Analysis and Machine Intelligence, IEEE Transactions
  on}, (1):1--8, 1980.

\bibitem{chen2011detection}
T.~Chen, S.~Jiang, L.~Chu, and Q.~Huang.
\newblock Detection and location of near-duplicate video sub-clips by finding
  dense subgraphs.
\newblock In {\em Proceedings of the 19th ACM international conference on
  Multimedia}, pages 1173--1176. ACM, 2011.

\bibitem{fisher1930genetical}
R.~Fisher.
\newblock The genetical theory of natural selection.
\newblock 1930.

\bibitem{groll2005new}
L.~Groll and J.~Jakel.
\newblock A new convergence proof of fuzzy $c$-means.
\newblock {\em Fuzzy Systems, IEEE Transactions on}, 13(5):717--720, 2005.

\bibitem{gunawardana2005convergence}
A.~Gunawardana and W.~Byrne.
\newblock Convergence theorems for generalized alternating minimization
  procedures.
\newblock {\em The Journal of Machine Learning Research}, 6:2049--2073, 2005.

\bibitem{gustafson1978fuzzy}
D.~Gustafson and W.~Kessel.
\newblock Fuzzy clustering with a fuzzy covariance matrix.
\newblock In {\em Decision and Control including the 17th Symposium on Adaptive
  Processes, 1978 IEEE Conference on}, volume~17, pages 761--766. IEEE, 1978.

\bibitem{hathaway1989relational}
R.~Hathaway, J.~Davenport, and J.~Bezdek.
\newblock Relational duals of the $c$-means clustering algorithms.
\newblock {\em Pattern recognition}, 22(2):205--212, 1989.

\bibitem{hofbauer1998evolutionary}
J.~Hofbauer and K.~Sigmund.
\newblock {\em Evolutionary games and population dynamics}.
\newblock Cambridge University Press, 1998.

\bibitem{hoppner2003contribution}
F.~Hoppner and F.~Klawonn.
\newblock A contribution to convergence theory of fuzzy $c$-means and
  derivatives.
\newblock {\em Fuzzy Systems, IEEE Transactions on}, 11(5):682--694, 2003.

\bibitem{istratescu2002fixed}
V.~Istratescu.
\newblock {\em Fixed point theory an introduction}, volume~7.
\newblock Kluwer Academic Print on Demand, 2002.

\bibitem{programmin1951h}
H.~W. Kuhn and A.~W. Tucker.
\newblock Nonlinear programming.
\newblock In {\em Proceedings, Second Berkeley Symposium in Mathematical
  Statistics and Probability, J. Neyman, editor. University of California Pres,
  Berkeley, Calif}, pages 481--92, 1951.

\bibitem{li2011graph}
X.~Li, A.~Dick, H.~Wang, C.~Shen, and A.~van~den Hengel.
\newblock Graph mode-based contextual kernels for robust svm tracking.
\newblock In {\em Computer Vision (ICCV), 2011 IEEE International Conference
  on}, pages 1156--1163. IEEE, 2011.

\bibitem{liu2010common}
H.~Liu and S.~Yan.
\newblock Common visual pattern discovery via spatially coherent
  correspondences.
\newblock In {\em Computer Vision and Pattern Recognition (CVPR), 2010 IEEE
  Conference on}, pages 1609--1616. IEEE, 2010.

\bibitem{liu2010robust}
H.~Liu and S.~Yan.
\newblock Robust graph mode seeking by graph shift.
\newblock In {\em International Conference on Machine Learning}, 2010.

\bibitem{pavan2003new}
M.~Pavan and M.~Pelillo.
\newblock A new graph-theoretic approach to clustering and segmentation.
\newblock In {\em Computer Vision and Pattern Recognition, 2003. Proceedings.
  2003 IEEE Computer Society Conference on}, volume~1, pages I--145. IEEE,
  2003.

\bibitem{pavan2007dominant}
M.~Pavan and M.~Pelillo.
\newblock Dominant sets and pairwise clustering.
\newblock {\em Pattern Analysis and Machine Intelligence, IEEE Transactions
  on}, 29(1):167--172, 2007.

\bibitem{selim1984k}
S.~Selim and M.~Ismail.
\newblock K-means-type algorithms: a generalized convergence theorem and
  characterization of local optimality.
\newblock {\em Pattern Analysis and Machine Intelligence, IEEE Transactions
  on}, (1):81--87, 1984.

\bibitem{weibull1997evolutionary}
J.~Weibull.
\newblock {\em Evolutionary game theory}.
\newblock The MIT press, 1997.

\bibitem{yang2011contour}
X.~Yang, H.~Liu, and L.~Jan~Latecki.
\newblock Contour-based object detection as dominant set computation.
\newblock {\em Pattern Recognition}, 2011.

\bibitem{yuan2011discovering}
J.~Yuan, G.~Zhao, Y.~Fu, Z.~Li, A.~Katsaggelos, and Y.~Wu.
\newblock Discovering thematic objects in image collections and videos.
\newblock {\em Image Processing, IEEE Transactions on}, (99):1--1, 2011.

\bibitem{zangwill1969nonlinear}
W.~Zangwill.
\newblock {\em Nonlinear programming: a unified approach}.
\newblock Prentice-Hall international series in management. Prentice-Hall,
  1969.

\bibitem{zhao2011robust}
J.~Zhao, J.~Ma, J.~Tian, J.~Ma, and D.~Zhang.
\newblock A robust method for vector field learning with application to
  mismatch removing.
\newblock In {\em Computer Vision and Pattern Recognition (CVPR), 2011 IEEE
  Conference on}, pages 2977--2984. IEEE, 2011.

\end{thebibliography}
% ACM needs 'a single self-contained file'!
%\balancecolumns
\appendix
%Appendix A
\section{Definition of $\Delta\boldsymbol{x}$}
According to \cite{liu2010robust}, $\Delta\boldsymbol{x}=t^*\boldsymbol{b}$,
$t^*$  and $\boldsymbol{b}$ are defined as:
\begin{equation} \label{eq:eq7}
\boldsymbol{b}= \left\{ \begin{array}{cc}
-\boldsymbol{x}_i s & i\in\sigma(\boldsymbol{x});\\
v_i, & i \notin \sigma(\boldsymbol{x}).
\end{array} \right.
\end{equation}

\begin{equation} \label{eq:eq8}
t^*= \left\{ \begin{array}{cc}
\frac{1}{s}, & \textrm{if}~\lambda s^2+2s\zeta-\omega\le 0\\
\textrm{min}(\frac{1}{s}, \frac{\zeta}{\lambda s^2+2s\zeta-\omega}), & \textrm{if}~\lambda s^2+2s\zeta-\omega> 0
\end{array} \right.
\end{equation}
where
\begin{equation}
v_i= \left\{ \begin{array}{cc}
0, & i\in\sigma(\boldsymbol{x})\\
\textrm{max}(a(\boldsymbol{x}, I_i)-g(\boldsymbol{x}), 0), & i \notin \sigma(\boldsymbol{x})
\end{array} \right.
\end{equation}
\begin{equation}
s=\sum_{i\notin\sigma(\boldsymbol{x})} v_i, \zeta = \sum_{i\notin\sigma(\boldsymbol{x})} v_i^2, \omega = \sum_{i,j} v_ia_{ij}v_j.
\end{equation}
$g(\boldsymbol{x}+\Delta\boldsymbol{x})-g(\boldsymbol{x})=-(\lambda s^2+2s\zeta-\omega)t^2+2\zeta t$.
When $\lambda s^2+2s\zeta-\omega\le0$, $g(\boldsymbol{x}+\Delta\boldsymbol{x})-g(\boldsymbol{x})\ge0$;
When $\lambda s^2+2s\zeta-\omega<0$, $t^*$ always lies in the interval $[0,  \frac{2\zeta}{\lambda s^2+2s\zeta-\omega}]$, which are the two solutions of $g(\boldsymbol{x}+\Delta\boldsymbol{x})-g(\boldsymbol{x})=0$, this leads to $g(\boldsymbol{x}+\Delta\boldsymbol{x})-g(\boldsymbol{x})>0$.

\section{Proof of Proposition 1}
\begin{proof}
To prove that the sequence set $\{\boldsymbol{x}_{(k)}\}_{k=0}^{\infty}\subset S$ is compact is equivalent to prove that the set is bounded and closed.
Since $(\boldsymbol{x}_1,\boldsymbol{x}_2,\cdots, \boldsymbol{x}_n)$ are all located in $[0, 1]$, the $\boldsymbol{x}$ value space is bounded. Also, from Equation (\ref{t_eq8}), $\boldsymbol{e}_i=(0,...,1,...,0)$, which means the $i$-th component of $\boldsymbol{x}$ is 1 and the others are 0, we can denote $\boldsymbol{x}(t)$ as:
\begin{equation}
\boldsymbol{x}(t)=\sum_{i=1}^n \boldsymbol{e}_i\cdot\frac{\boldsymbol{x}(t-1)_i(A\boldsymbol{x}_i(t-1))}{\boldsymbol{x}(t-1)^T A\boldsymbol{x}(t-1)}
\end{equation}
Since $\sum_{i=1}^n \frac{\boldsymbol{x}(t-1)_i(A\boldsymbol{x}_i(t-1))}{\boldsymbol{x}(t-1)^T A\boldsymbol{x}(t-1)}
=1$ and $0\le\frac{\boldsymbol{x}(t-1)_i(A\boldsymbol{x}_i(t-1))}{\boldsymbol{x}(t-1)^T A\boldsymbol{x}(t-1)}\le 1,i=1,\ldots,n.$. Adding $\Delta\boldsymbol{x}$ (defined in Appendix A) still holds the expression $\boldsymbol{x}(t)=\sum_{i=1}^n \boldsymbol{e}_i\cdot y_i, \sum_{i=1}^n y_i=1$, thus $\boldsymbol{x}(t)$ is in the convex hull of $S$, so it is closed.
Therefore, the sequence set $\{\boldsymbol{x}_{(k)}\}_{k=0}^{\infty}\subset S$ is both bounded and closed.
\end{proof}

\section{Proof of Proposition 2}
\begin{proof}
This proposition is known in mathematical biology as the fundamental theory of natural selection \cite{hofbauer1998evolutionary} and, in its original form, we can trace it back to \cite{fisher1930genetical}.

We can also prove that Proposition \ref{t_prop2} is a special case of Baum-Eagon inequality (\cite{baum1967inequality}).
%Due to space limit, we ignore the details here.
 We denote $x_i$ as $x_i = \prod_{j=1}^n x_j^{\mu_{ij}}$, here $
 \mu_{ij}= \left\{ \begin{array}{cc}
 1, & i=j;\\
 0, & i\neq j.
 \end{array} \right.
 $. Thus we wish to prove that when $\boldsymbol{x}(t)\notin\Gamma$:
 \begin{equation}
 \begin{split}
 g(x)=&x^TAx=\sum_{i=1}^n \omega_i x_i=\sum_{i=1}^n \omega_i\prod_{j=1}^n x_j^{\mu_{ij}}\\
 &<\sum_{i=1}^n\omega_i\prod_{j=1}^n J(x_j)^{\mu_{ij}}.
 \end{split}
 \end{equation}

 From H\"older inequation and $x_i^2=x_i\cdot\prod_{j=1}^n x_j^{\mu_{ij}}$, we can get the result as:
 \begin{equation}
 \begin{split}
 g(x)&=\sum_{i=1}^n \{\omega_i\cdot\prod_{j=1}^n J(x_j)^{\mu_{ij}}\}^{\frac{1}{2}}\times\{\omega_{i}^{\frac{1}{2}}x_i\prod_{j=1}^{n}(\frac{1}{J(x)})^{\frac{\mu_{ij}}{2}}\}\\
 & \le \{\sum_{i=1}^n \omega_i\prod_{j=1}^n J(x_j)^{\mu_{ij}}\}^{\frac{1}{2}}\times\{\sum_{i=1}^n \omega_i x_i \prod_{j=1}^n (\frac{x_j}{J(x_j)})^{\mu_{ij}}\}^{\frac{1}{2}}
 \end{split}
 \end{equation}
 Equality holds if and only if $\forall p,q\in\{1,\cdots, n\}$,
 \begin{equation}
 \begin{split}
 &\frac{\{\omega_{p}\cdot\prod_{j=1}^n J(x_j)^{\mu_{{p}j}}\}^{\frac{1}{2}}}{\omega_{p}^{\frac{1}{2}}x_{p}\prod_{j=1}^{n}(\frac{1}{J(x)})^{\frac{\mu_{{p}j}}{2}}}
 =\frac{\{\omega_{q}\cdot\prod_{j=1}^n J(x_j)^{\mu_{{q}j}}\}^{\frac{1}{2}}}{\omega_{q}^{\frac{1}{2}}x_{q}\prod_{j=1}^{n}(\frac{1}{J(x)})^{\frac{\mu_{{q}j}}{2}}}\\
 \iff & \frac{J(x_{p})}{x_{p}}=\frac{J(x_{q})}{x_{q}} \iff \omega_{p}=\omega_{q}
 \end{split}
 \end{equation}
 Using the inequality of geometric and arithmetic means to the double products of the second brace, we can conclude:
 \begin{equation}
 \begin{split}
 &\sum_{i=1}^n \omega_i x_i \prod_{j=1}^n (\frac{x_j}{J(x_j)})^{{\mu_{ij}}}\le\sum_{i=1}^n \omega_i x_i \sum_{j=1}^n {\mu_{ij}}\cdot\frac{x_j}{J(x_j)}\\
 &= \sum_{i=1}^n \omega_i x_i \sum_{j=1}^n \mu_{ij}x_j\cdot\frac{\sum_{k=1}^n \omega_kx_k}{\omega_jx_j}\\
 &=\sum_{k=1}^n \omega_k x_k \cdot\sum_{j=1}^n x_j\cdot\frac{\sum_{i=1}^n \omega_i x_i\cdot\mu_{ij}}{\omega_j x_j}\\
 &=\sum_{k=1}^n \omega_k x_k \cdot\sum_{j=1}^n x_j= \sum_{k=1}^n \omega_k x_k.
 \end{split}
 \end{equation}
 Equality holds if and only if $\forall p,q\in\{1,\cdots, n\}$,
 \begin{equation}
 \frac{x_p}{J(x_p)}=\frac{x_p}{J(x_p)} ~\iff~ \omega_p = \omega_q.
 \end{equation}
 Here the last equation succeed because $\omega_ix_i\cdot\mu_{ij}= \omega_jx_j$ if and only if $j=i$, otherwise it is $0$, and $\sum_{j=1}^n x_j=1.$\\

 we put the result into the second braces, and get
 \begin{equation}
 \begin{split}
  & \sum_{i=1}^n \omega_i x_i\le \{\sum_{i=1}^n \omega_i \prod_{j=1}^n J(x_j)^{\mu_{ij}}\}^{\frac{1}{3}}\times\{\sum_{i=1}^n \omega_ix_i\}^{\frac{2}{3}}\\
 \iff & \{\sum_{i=1}^n \omega_ix_i\}^{\frac{1}{3}}\le\{\sum_{i=1}^n \omega_i \prod_{j=1}^n J(x_j)^{\mu_{ij}}\}^{\frac{1}{3}} \\
 \iff & \sum_{i=1}^n \omega_ix_i\le\sum_{i=1}^n \omega_i \prod_{j=1}^n J(x_j)^{\mu_{ij}}
 \end{split}
 \end{equation}
 Equality holds if and only if $\omega_i=\omega_j, \forall i,j\in\{1, \cdots, n\}$. It is the situation contained by the solution set $\Gamma$.\\
 Thus, the function $g(x)=\boldsymbol{x}^tA\boldsymbol{x}$ is strictly increasing along any nonconstant trajectory of (\ref{t_eq8}) when $\boldsymbol{x}\in X/\Gamma$.
 \end{proof}

\section{Proof of Proposition 4}
\begin{proof}
The continuity of $g(\boldsymbol{x})=\boldsymbol{x}^T A\boldsymbol{x}$ is obvious. We here discuss the increasing monotonicity.
From Equation (\ref{t_eq7}), we have
\begin{equation}
g(\boldsymbol{x})\le g(C(\boldsymbol{x}))<g(B\circ C(\boldsymbol{x}))=g(T_m(\boldsymbol{x}))
\end{equation}
\end{proof}

\section{Proof of Proposition 5}
\begin{proof}
If $\boldsymbol{x}^0\in X$, then $\boldsymbol{x}^n\to \boldsymbol{x}^0$ and $\boldsymbol{y}^n\to\boldsymbol{y}^0$ when $n\to\infty$. Consequently $\boldsymbol{y}^n\in C(\boldsymbol{x}^n)$, indicating that
\begin{equation}
\boldsymbol{y}^n = \boldsymbol{x}^n + \Delta\boldsymbol{x}^n
\end{equation}
and we need to prove that
\begin{equation}
\boldsymbol{y}^0 = C(\boldsymbol{x}^0)= \boldsymbol{x}^0 + \Delta\boldsymbol{x}^0.
\end{equation}
This lies in two situations:
\begin{enumerate}
\item[(1)] $\boldsymbol{x}^0$ is in $\Gamma$, thus $\Delta\boldsymbol{x}^0\ = 0$.
\item[(2)] $\boldsymbol{x}^0$ is in $X\backslash \Gamma$.
\end{enumerate}
%$\Delta\boldsymbol{x}^0=t^{*0}\cdot\boldsymbol{b}^0$, since $\boldsymbol{x}^0$ is fixed, we got $v_i^0, s^0, \zeta^0, \omega^0$ fixed, so does $\boldsymbol{b}^0, t^0$. \\
For situation (2), as $\boldsymbol{x}^n\to\boldsymbol{x}^0$ and $\boldsymbol{y}^n\to\boldsymbol{y}^0$, we can find a large $N$ such that $\forall \varepsilon > 0, \exists n>N,~\textrm{so that}~|\boldsymbol{x}^n - \boldsymbol{x}^0|<\varepsilon, |\boldsymbol{y}^n - \boldsymbol{y}^0|<\varepsilon$. Consequently, we have
\begin{equation}
\begin{split}
&|\boldsymbol{y}^0-C(\boldsymbol{x}^0)|\le|\boldsymbol{y}^0-\boldsymbol{y}^n|+|\boldsymbol{y}^n-C(\boldsymbol{x}^n)|+|C(\boldsymbol{x}^n)\\
&-C(\boldsymbol{x}^0)| \le\varepsilon + |\boldsymbol{x}^n - \boldsymbol{x}^0|+|\Delta\boldsymbol{x}^0-\Delta\boldsymbol{x}^n|\\
& \le 2\varepsilon + |\Delta\boldsymbol{x}^0-\Delta\boldsymbol{x}^n|\\
\end{split}
\end{equation}
According to Lemma \ref{eq:lemma2} and the definitions of $\boldsymbol{b}$ and $t$, it is a continuous mapping on $\boldsymbol{x}$, this results in that $\forall\varepsilon>0, \exists\delta$ such that $|\Delta\boldsymbol{x}^0-\Delta\boldsymbol{x}^n|\le \varepsilon$.
So, $\forall\varepsilon>0$, if $N_1$ is big enough, $n>N_1, |\boldsymbol{x}^n - \boldsymbol{x}^0|<$min$(\varepsilon/3, \delta)$, $|\boldsymbol{y}^n - \boldsymbol{y}^0|<\varepsilon/3$, then
\begin{equation}
|\boldsymbol{y}^0-C(\boldsymbol{x}^0)|\le\varepsilon
\end{equation}
Hence $\boldsymbol{y}^0=C(\boldsymbol{x}^0)$, and $C$ is closed on $X\backslash \Gamma$.
\end{proof}
\balancecolumns
% That's all folks!

\end{document}